\documentclass[11pt]{article}
\usepackage{multirow, comment}
\usepackage{graphicx}	
\usepackage{amsfonts,amsmath,amssymb,mathrsfs,amsthm}
\usepackage{algorithm}
\usepackage{algorithmic}
\usepackage{color}
\usepackage{natbib}
\usepackage[utf8]{inputenc}
\usepackage[title,page]{appendix}
\usepackage[margin=2.5cm]{geometry}
\usepackage{parskip} 

\newtheorem{definition}{Definition}[section]

\newtheorem{proposition}[definition]{Proposition}

\newcommand{\tmu}{\tilde\mu}
\newcommand{\tn}{\tilde{n}}

\newcommand{\tV}{\tilde{V}}

\DeclareSymbolFont{bbold}{U}{bbold}{m}{n}
\DeclareSymbolFontAlphabet{\mathbbold}{bbold}

\title{Practical Calculation of Gittins Indices for Multi-armed Bandits}
\author{James Edwards\\ Department of Mathematics and Statistics, Lancaster
University\\j.edwards4@lancaster.ac.uk}
\date{February 27, 2019}

\begin{document}
\maketitle
\begin{abstract}
Gittins indices provide an optimal solution to the classical multi-armed bandit
problem. An obstacle to their use has been the common perception that their
computation is very difficult. This paper demonstrates an accessible general
methodology for the calculating Gittins indices for the multi-armed bandit with
a detailed study on the cases of Bernoulli and Gaussian rewards. With
accompanying easy-to-use open source software, this work removes computation as
a barrier to using Gittins indices in these commonly found settings.
\end{abstract}
\textbf{Keywords:} Multi-armed bandits; Gittins index; Stochastic Dynamic
programming

\section{Introduction}
\label{sec:intro_GI}
The Gittins index (GI) is known to provide a method for a Bayes optimal solution
to the multi-armed bandit problem (MAB)
(\citealt{gittins1979bandit,gittins2011multi}).
In addition, Gittins indices (GIs) and their generalisation Whittle indices have
been shown to provide strongly performing policies in many related problems even
when not optimal (\citealt{whittle1980multi}). 

The breakthrough that GIs provided was one of computational tractability since
previously, optimal methods were only practical for very restricted set of small
MAB problems. However, GIs are often not used for the MAB, nor many other
problems for which they are well suited, despite the rarity of other tractable
optimal or near optimal solutions. A major reason is the perception that GIs are
hard to compute in practice:  ``\ldots the lookahead approaches [GI] become
intractable in all but the simplest setting\ldots''
(\citealt{may2012optimistic}) and, ``Logical and computational difficulties have
prevented the widespread adoption of Gittins indices''
(\citealt{scott2010modern}). While it is true that in some settings computation
will not be practical, for common forms of the MAB with a range of standard
reward distributions calculation of GIs is certainly tractable. 

There is a need, though, to make the practice of GI calculation clearer and more
accessible. \cite{powell2007approximate} observed ``Unfortunately, at the time
of this writing, there do not exist easy to use software utilities for computing
standard Gittins indices''. This remains the case today. \cite{gittins2011multi}
provides tables of GI values for some problem settings and some MATLAB code, but
these are limited in scope. \cite{lattimore2016regret} calculates indices for a
finite horizon undiscounted MAB with the C++ code made available but this is
limited to Gaussian rewards only, with fixed observation noise (allowed
to vary here).

This paper describes a general methodology, using stochastic dynamic
programming, for calculating GIs. Details for the method are provided for the
MAB with Bernoulli and normal rewards (respectively BMAB and NMAB). This builds
on work in chapter 8 of \cite{gittins2011multi}, adapting and formalising the
method given there to be more accessible with a more detailed and general
implementation. Convergence tests give accuracy and calculation times for
appropriate settings. New results exploiting monotonicity in the dynamic
programme bring large improvements in memory use and speed.

The central contribution of this work is the accompanying open source code that
has been developed in the R programming language and is available at
\texttt{https://github.com/jedwards24/gittins}. The code and this
paper, enables wider use of GIs in both application and
research by making their calculation for the BMAB and NMAB accessible,
efficient, and easily reproducible.

The rest of this section briefly describes the MAB and how GIs can be used to
give an optimal solution. Section \ref{sec:solution} describes a general method
for calculating GIs. Sections \ref{sec:bmab} and \ref{sec:nmab} then give
details of GI calculation for, respectively, the BMAB and NMAB. Section
\ref{sec:gi_discuss} discusses outstanding issues and extensions to
other MAB problems. Reported calculation times used a Intel Xeon E5-1630v4
3.7GHz processor with 64GB RAM without parallelisation.

\subsection{Problem Definition}
The motivating problem is the classical Bayesian MAB. The notation used assumes
reward distributions from the exponential family, as described in
\cite{edwards2017identification}, but the problems and solution
framework is appropriate for general reward distributions.

At each time $t=0,1,2,\ldots$ an \textit{arm} $a_t\in\{1,\ldots,k\}$ is chosen.
Associated with each arm $a$ is an unknown parameter $\theta_a$ and choosing arm
$a$ at time $t$ results in an observation (or \textit{reward}) $y_t$ drawn from
a density $f(y_t\mid\theta_a)$ which, apart from the parameter $\theta_a$, has known form.
Our belief in the value of $\theta_a$ is given by $g(\theta_a\mid\Sigma_a,n_a)$
where $\Sigma_a$ and $n_a$ are known hyperparameters representing,
respectively, for arm $a$, the Bayesian sum of rewards and the Bayesian number
of observations. At each time the current value of $\Sigma_a$ and $n_a$ give the
\textit{informational state} of the arm. After an observation $y_t$ from arm
$a$, Bayesian updating produces a posterior belief for arm $a$ of
$g(\theta_a\mid\Sigma_a+y_t,n_a+1)$, with beliefs for all other arms unchanged.
 
The total return is the discounted sum of rewards
$\sum\nolimits_{t=0}^{\infty}\gamma^{t}y_{t}$. The objective is to design a
policy (a rule for choosing arms) to maximise the \textit{Bayes' return},
namely the total return averaged over both realisations of the system and prior
information.

The Bayes' return is maximised (\citealt{gittins2011multi}) by the
\textit{Gittins Index policy} which chooses, in the current state, any arm
$a$, satisfying
\begin{equation}
\nu^{GI}(\Sigma _{a},n_{a},\gamma)=\max_{1\leq b\leq k}\nu ^{GI}(\Sigma
_{b},n_{b},\gamma )\,, 
\end{equation}
where $\nu^{GI}$ is the GI which will be defined in Section
\ref{sec:solution}. For a constant discount factor GI values are independent of
time.

Therefore an optimal policy, and hence a solution to the MAB, is given by the GI
of all possible arm states. The efficient calculation of these is the problem
addressed in this paper. The arm subscripts $a$ will often be dropped since we
are only interested in a single arm state for a given index calculation.

\subsection{Using a GI-based policy in Practice}
\label{sec:policy}
A GI-based policy requires the calculation of a GI for each arm state.
The optimality of the policy depends on the accuracy to which the GIs are
calculated but, as long as reasonable accuracy is used, any suboptimality is limited since
(i) decisions involving arms that have very similar GI values will be rare, and
(ii) the cost of a suboptimal action is small when the GI of the suboptimal arm
is close to that of the optimal arm: ``\ldots an approximate largest-index rule
yields an approximate optimal policy'' (\citealt{katehakis1987multi}).
\cite{glazebrook1982evaluation} gives a bound for the lost reward of a
suboptimal policy in terms of GIs so by bounding GI accuracy we can
bound lost reward.

GIs can be calculated either online, as they are needed, or offline where they
are stored then retrieved when needed. Online calculation has the advantage that
GIs need only be calculated for states that are visited which is beneficial when
the state space is large (especially if continuous). Additionally,
only the GI of the arm selected need be recalculated since the GIs of
other arms are unchanged from the previous time.

However, for many applications, online calculation will be too slow to be
used to reasonable accuracy so this paper will focus on offline calculation,
although the methods are applicable to online use. Generally, offline
calculation will be effective whenever online calculation is practical, the only
possible exception being with large state spaces. 

For offline calculation we theoretically need to find in advance the GI for each
state that an arm in the MAB may visit. This may not be possible as the
state space can be infinite, either due to the MAB's infinite time horizon, or a
continuous state space resulting from continuous rewards or priors parameterised
over a continuous range. However, in practice we only need GIs for states that
can be states reached after a suitable finite time $T$. Discounted rewards
become very small for large times and states reached after $T$ observations on a
single arm $a$ will have tight belief distributions so the mean reward
$\Sigma_a/n_a$ becomes a good approximation for $\nu^{GI}(\Sigma_a,n_a,\gamma)$.

The issue of continuous state spaces can be solved by using monotonicity
properties of $\nu^{GI}$ to bound GI values for any state using the GIs
of similar states. From \cite{edwards2017identification},
$\nu^{GI}(c\Sigma,cn,\gamma)$ is decreasing in $c\in\mathbb{R}^{+}$ for any
fixed $\Sigma,n,\gamma$ and is increasing in $\Sigma $ for any fixed
$c,n,\gamma$. With this result $\nu^{GI}(\Sigma,n,\gamma)$ for a discrete grid of $\Sigma$
and $n$ can be used to bound and interpolate $\nu^{GI}$ for any interior
state. From \cite{kelly1981multi}, $\nu^{GI}(\Sigma,n,\gamma)$ is non-decreasing
in $\gamma$ which enables similar interpolation for $\gamma$, if needed.

\section{General Method of Calculation}
\label{sec:solution}
Various methods exist for calculating GIs (for a review see
\citealt{chakravorty2014multi} or \citealt{gittins2011multi}) but we will use
\emph{calibration}. Calibration uses a bandit process with a \emph{retirement option}
(\citealt{whittle1980multi}) which is sometimes referred to as a one-armed
bandit (OAB). The single arm in question is the one for which we wish to find
the GI which we will call the \emph{risky arm}. At each time we have the choice
to \emph{continue} to play this arm or instead choose an arm of known fixed
reward $\lambda$ (the \emph{safe arm}). Since the safe arm does not
change, once it is optimal to choose it then it will continue to be optimal
indefinitely (this is retirement). The GI of the risky arm is the value of
$\lambda$ for which, at the start of the OAB, we are indifferent between
choosing the safe or risky arms. Hence we must find the expected reward (or
\emph{value}) of the OAB for a given $\lambda$ with both initial actions.

Let $Y$ be a random variable of the predictive distribution of the
observed reward from the risky arm with state $(\Sigma,n)$. Then the value
function for the OAB given $\lambda$, $\gamma$, and the risky arm's state
$(\Sigma, n)$ is
\begin{align}   
\label{equ:gi_oab_dp}
V(\Sigma ,n,\gamma,\lambda)=\max\left\{\frac{\Sigma }{n}+
\gamma\mathbb{E}_{Y}\big[V(\Sigma+Y,n+1,\gamma,\lambda)\big];\frac{\lambda}{1-\gamma}\right\}.
\end{align}
The first (recursive) part in the maximisation is the value of risky arm. The
value of the safe arm, $\lambda/(1-\gamma)$, is found directly due to
retirement. We only need the sign, not the absolute value, of $V$ so
$V(\Sigma,n,\gamma,\lambda)$ will instead give the relative value between the
safe and risky arms and \eqref{equ:gi_oab_dp} is replaced by the simpler
\begin{equation}
\label{equ:gi_oab_dp2}
V(\Sigma,n,\gamma,\lambda)=\max\left\{\frac{\Sigma }{n}-\lambda+
\gamma\mathbb{E}_{Y}\big[V(\Sigma +Y,n+1,\gamma ,\lambda) \big]
;0\right\}.
\end{equation}
The Gittins index can then be defined as
\begin{equation}
\label{equ:gi_gi_def}
\nu^{GI}(\Sigma,n,\gamma)=\min\{\lambda:V(\Sigma,n,\gamma,\lambda)=0\} . 
\end{equation}
To find $\lambda$ satisfying \eqref{equ:gi_gi_def} a numerical method must be used.
Observe that $V(\Sigma,n,\gamma,\lambda)$ is decreasing in $\lambda$ for fixed
$\Sigma$, $n$ and $\gamma $ so we can progressively narrow an
interval containing $\lambda$ by repeatedly finding
$V\left(\Sigma,n,\gamma,\hat\lambda\right)$ for appropriate $\hat\lambda$. The
general method for a single state is given in Algorithm \ref{alg:GI_calibration}.
\begin{algorithm}[H]                  
\caption{Calibration for Gittins indices}          
\label{alg:GI_calibration}                           
\begin{algorithmic}                    
    \REQUIRE Parameters $\Sigma$, $n$, $\gamma$. Initial bounds for
    $\nu^{GI}(\Sigma$, $n$, $\gamma)$ given by $u$, $l$. A required
    accuracy $\epsilon$.
    \WHILE{$u-l>\epsilon$}
     \STATE{$\hat\lambda\gets (l+u)/2$}
     \STATE{Calculate $V(\Sigma,n,\gamma,\hat\lambda)$ as given in
     \eqref{equ:gi_oab_dp2}}
     \IF {$V(\Sigma,n,\gamma,\hat\lambda)>0$}
       \STATE{$l\gets\hat\lambda$}
     \ELSE
       \STATE{$u\gets\hat\lambda$}
     \ENDIF
    \ENDWHILE
    \ENSURE An interval $[l,u]$ which contains
    $\nu^{GI}(\Sigma$, $n$, $\gamma)$ where $u-l<\epsilon$ .
\end{algorithmic}
\end{algorithm}
The algorithm initialises an interval in which $\nu^{GI}(\Sigma$, $n$, $\gamma)$
is known to lie (methods for doing this will be given shortly). The interval is
then reduced in size, using bisection and repeated calculation of the value function in \eqref{equ:gi_oab_dp2}, until it is
sufficiently small for our purposes. The mid-point of this interval gives
$\nu^{GI}(\Sigma$, $n$, $\gamma)$ within the desired accuracy $\epsilon$ of the
true value.

Other interval reduction methods can be used but bisection works well.
With an initial interval $[l,u]$ the number of calculations of $V(\Sigma,n,\gamma,\hat\lambda)$ required is
$N_V=\left\lceil\log\left(\frac{\epsilon}{u-l}\right)/\log(0.5)\right\rceil$.

Bounds for $\nu^{GI}(\Sigma$, $n$, $\gamma)$ are needed to initialise an
interval for Algorithm \ref{alg:GI_calibration} and tighter bounds will reduce
computation time.  Two numerical but fast-to-calculate bounds are used in the
software. The lower bound is the \textit{knowledge gradient index} from
\cite{edwards2017identification} which uses a one-step lookahead approximate
solution to the OAB. The upper bound is a similar approximate solution to the
OAB which assumes that all information about the system is revealed after a
single stage. Details on both of these and other bounds as well as their relative
tightness can be found in \cite{edwards2016exploration}.

Each $V(\Sigma,n,\gamma,\hat\lambda)$ in Algorithm \ref{alg:GI_calibration} is
found with stochastic dynamic programming using the recursive equation
\eqref{equ:gi_oab_dp2}. Details of this calculation for the
BMAB and NMAB will be given in Sections \ref{ssec:GI_value_bmab} and
\ref{ssec:GI_value_nmab}.

Note that we now have two state spaces: that of the arms in the original MAB
and that of the risky arm in the OAB dynamic programme. To distinguish
between these the term \textit{stage} will be used for the OAB in place of time.
The arm state $(\Sigma_a,n_a)$ in the MAB forms the initial state of the risky
arm at stage 0 in the OAB which then evolves through states $(\Sigma,n)$.

The backward recursion process of dynamic programming cannot proceed for an
infinite number of stages so a finite horizon approximation is used. An $N$ is
chosen and the values of \textit{terminal states}
$\{(\Sigma,n,\gamma,\lambda):n=n_a+N\}$ at stage $N$ are calculated directly using some approximation. An effective
approximation uses the expected reward at stage $N$ assuming no further
learning:
\begin{equation}
\label{equ:gi_V_N}
V_N(\Sigma,n,\gamma,\lambda)=\frac{\gamma^N}{1-\gamma}\max(\Sigma/n-\lambda,0).
\end{equation}
States at stages prior to stage $N$ are then found recursively using
\eqref{equ:gi_oab_dp2}, calculating backwards through stages $N-1$, $N-2$, etc.
Discounting ensures that this index with a large $N$ gives a very good
approximation which can be made arbitrarily close to $\nu^{GI}$ by increasing
$N$.

The size of the dynamic program and hence the computation required depends on
$N$ so it is desirable to use as small a value as possible that gives the
required accuracy. The approximation error is bounded above by the remaining
reward at stage $N$. However, a study with different $N$ in
Appendix \ref{sec:finite_N} shows that the error is much smaller than this and
quite small values of $N$ are sufficient. The reason is as follows. If the safe arm is the optimal
action at time $N$ then the approximation is exact, while if the risky arm is optimal then it will have
been chosen $N$ times and $\Sigma_N/n_N$ will be a good estimate of its true
expected reward. This and discounting means there is little approximation error.

\section{Calculation Details with Bernoulli Rewards}
\label{sec:bmab}
This section will give details on the calculation of GI for the BMAB
where observed rewards $f(y\mid\theta)\sim Bern(\theta)$. Section
\ref{ssec:GI_value_bmab} will details the OAB solution for a single state.
Calculation of GI for multiple states can be done by repeated use of single
state calculations but Section \ref{ssec:gi_multiple_bmab} considers more
efficient methods.

\subsection{Value Function Calculation - BMAB}
\label{ssec:GI_value_bmab}
The application of dynamic programming in the BMAB case to find
$\nu^{GI}(\Sigma_a,n_a,\gamma)$ is largely straightforward due to binary
outcomes. The (discrete) OAB state space is
$\{(\Sigma,n):\Sigma=\Sigma_a,\Sigma_a+1,\ldots,\Sigma_a+N,
n=n_a,n_a+1,\ldots,n_a+N; \Sigma\leq n\}$, a total of $\frac{1}{2}(N+2)(N+1)$
states. The belief distribution for $\theta$ is $Beta(\Sigma, n-\Sigma)$ so the
predicted probability of a success and the immediate expected reward is
$\Sigma/n$. The value function \eqref{equ:gi_oab_dp2} then becomes
\begin{align*}
V(\Sigma&,n,\gamma,\lambda)\\
&=\max\left\{
\frac{\Sigma}{n}-\lambda+\gamma\left[\frac{\Sigma}{n}V(\Sigma+1,n+1,\gamma,\lambda)
+\left(1-\frac{\Sigma}{n}\right)V(\Sigma,n+1,\gamma,\lambda)\right];0\right\},
\end{align*}
with terminal states as given in \eqref{equ:gi_V_N}. Apart from the use of a
finite $N$ (see \ref{sec:finite_N}), the calculation of
$\nu^{GI}(\Sigma_a,n_a,\gamma)$ for the BMAB is therefore exact. 

\subsection{Multiple State Computation - BMAB}
\label{ssec:gi_multiple_bmab}
For the BMAB there are two dimensions, $\Sigma$ and $n$, for each $\gamma$.
Outcomes $y$ are in $\{0,1\}$ so for priors $\Sigma_0$, $n_0$ we need GIs
for MAB states
\begin{equation}
\label{equ:gi_bmab_states}
\{(\Sigma,n):\Sigma=\Sigma_0,\Sigma_0+1,\ldots,\Sigma_0+T,n=n_0,\ldots,n_0+T,\Sigma\leq
n\}.
\end{equation}
If a set of GI values is needed for any possible prior (which hypothetically
could take any positive value) then this can be done by finding GIs for the
set of states in \eqref{equ:gi_bmab_states} for a reasonable grid of
$\{(\Sigma_0,n_0):\Sigma_0\in(0,1)$,$n_0\in(0,2],\Sigma_0<n_0\}$ and
interpolating where needed. The interpolation is best done offline to create a
two-dimensional matrix of values for each arm that has distinct priors.

An alternative to finding GIs one state at a time as described (the \emph{state
method}) is to use the method given in Section 8.4 of \cite{gittins2011multi}
(referred to here as the \emph{block method}). The block method finds GI values
for a block of states in one go by stepping through an increasing sequence of
values of $\lambda$ and assigning index values to states when the safe arm is
first preferred to the risky arm. By doing this it reuses some value function
calculations and can therefore be more efficient if used on a large number of
states. 

However, the state method has advantages for general use. Firstly, computation
time scales linearly with accuracy for the block method but logarithmically for
the state method. Therefore, the block method tends to be faster for low
accuracy (when both methods are fast) but slower for higher accuracy. GIs for
the whole state space given by \eqref{equ:gi_bmab_states} with $T=100$, $N=200$
and $\epsilon=10^{-4}$ (5151 states) took just over 4 minutes using the state
method, with block method six time slower. Secondly, the state method
parallelises naturally, which can dramatically reduce the time needed. If using
parallelisation states should be assigned to processors in a manner which uses
GIs of neighbouring states to initialise starting intervals for new GIs. For
example, send all states with the same $\Sigma$ to one processor in order of ascending $n$.

\section{Calculation Details with Normal Rewards}
\label{sec:nmab}
This section will give details on the calculation of GI for the NMAB. The
standard version of the NMAB has observations $f(y\mid\theta)\sim N(\theta,1)$
but the method here allows a more general version where each arm has an extra
parameter $\tau_a>0$ for the precision of observations so that
$f(y\mid\theta,\tau)\sim N(\theta,1/\tau_a)$. Each $\tau_a$ is assumed to be
known. In addition, the value function calculation for the NMAB is easier to
describe using a reparameterisation of states to $(\mu,n)$ where $\mu=\Sigma/n$.
This version of GI will be denoted $\nu^{GI(\tau)}(\mu,n,\gamma,\tau)$.

Calculation of GIs for the NMAB is much more challenging than for the BMAB due
to continuous state spaces for both the MAB and the OAB. Section
\ref{ssec:gi_multiple_nmab} first shows how the MAB space can be simplified then
Section \ref{ssec:GI_value_nmab} will detail the calculation for each state.
\subsection{Multiple State Computation - NMAB}
\label{ssec:gi_multiple_nmab}
The NMAB takes continuous outcomes so the arm state space is potentially
continuous in both $n$ and $\mu$ dimensions. Fortunately,
\cite{gittins2011multi} (p 217) gives invariance properties of $\nu^{GI(\tau)}$
for the NMAB from which
\begin{equation}
\label{equ:gi_nmab_transform}
\nu^{GI(\tau)}(\mu,n,\gamma,\tau)=\frac{\mu}{n}+\frac{1}{\sqrt\tau}
\nu^{GI(\tau)}\left(0,\frac{n}{\tau},\gamma,1\right).
\end{equation}
This reduces the problem to a single dimension for each $\gamma$ with only
calculations using $\mu=0$ and $\tau=1$ needed. For a given $\gamma$, $\tau$,
$T$ and prior $n_0$ we then require calculations of $\nu^{GI(\tau)}(0,n,\gamma,1)$ for 
\begin{align*}
n&=\frac{n_0}{\tau},\frac{n_0+\tau}{\tau},\ldots,\frac{n_0+T\tau}{\tau}=\frac{n_0}{\tau},\frac{n_0}{\tau}+1,\ldots,\frac{n_0}{\tau}+T.
\end{align*}
So with $n_0=1$ we would need values for $ n=10,11,\ldots,10+T$ for $\tau=0.1$
and $ n=0.1,1.1,\ldots,0.1+T$ for $\tau=10$. Note that the $1/\sqrt{\tau}$ in
\eqref{equ:gi_nmab_transform} inflates errors when values are transformed for
$\tau<1$ so untransformed values will need to be calculated to a higher accuracy
than if used directly. 

In addition to the monotonicity properties in Section
\ref{sec:policy}, \cite{yao2006some} gives monotonicity results specific to the
NMAB: $\nu^{GI(\tau)}(\mu,n,\gamma,\tau)$ is non-decreasing in $\mu$ and
$\tau$ and non-increasing in $n$. This enables narrower starting intervals
for OAB calculations if GIs are calculated sequentially in $n$.

\subsection{Value Function Calculation - NMAB}
\label{ssec:GI_value_nmab}
The introduction of $\tau$ and reparameterisation with $\mu$ requires a new
value function to replace \eqref{equ:gi_oab_dp2}. Let $Y\sim
N\left(\mu,\frac{1}{n}+\frac{1}{\tau}\right)$ be the predictive distribution of
the observed reward $y$ and $\mu^+=(n\mu+\tau y)/(n+\tau)$ be the posterior of $\mu$ after an observation $y$. Then, for this
section, \eqref{equ:gi_oab_dp2} is replaced by
\begin{equation}
\label{equ:value_tau}
V(\mu,n,\gamma,\tau,\lambda)=\max\left\{\mu-\lambda+
\gamma\mathbb{E}_{Y}\big[V(\mu^+,n+\tau,\gamma,\tau ,\lambda) \big]
;0\right\}.
\end{equation}

Continuous, unbounded rewards make calculation of $V$ much more difficult than
for the BMAB. The OAB process evolves in two dimensions, $\mu$ and $n$. With
fixed $\tau$, the $n$ dimension takes discrete values but the $\mu$ dimension is
continuous and so must be discretised and bounded to ensure a finite number of
states in the dynamic programme. This is done with two new parameters. The
first, $\xi$, describes the extent of the state space and is the number of
standard deviations $\sigma=\sqrt{1/n}$ of $g(\theta\mid\Sigma, n)$ from $\mu_a$
included in the $\tmu$ dimension. The second, $\delta$, controls the fineness of
the discretisation. In addition we will restrict $\mu\geq\mu_a$ which will be
justified shortly. Therefore $\Omega$, the range for $\tmu$ is
\[ 
\Omega=\{\mu:\mu=\mu_a,
\mu_a+\delta,\mu_a+2\delta,\ldots,\mu_a+\left\lceil\frac{\xi\sigma}{\delta}\right\rceil\delta\},
\]
and the full OAB state space is $\{(\mu,n):\mu\in\Omega;n=n_a,n_a+\tau,
n_a+2\tau,\ldots,n_a+N\tau\}$. The total number of states is therefore
$(N+1)(\left\lceil\xi\sigma/\delta\right\rceil+1)\approx N\xi\sigma/\delta$.

The immediate reward of each state is $\mu$ and the values of
the states at stage $N$ are found using \eqref{equ:gi_V_N} as with the BMAB.
Let $\pi((\mu,n),\mu^+,\tau)$ be the transition probability from any state
$(\mu,n)$ to another state $(\mu^+,n+\tau)$, $\mu^+\in\Omega$. Then
$\pi((\mu,n),\mu^+,\tau)$ is given by the probability that the posterior mean is
in the interval $[\mu^+-\delta/2,\mu^++\delta/2)$. Thus we have
\[
\pi((\mu,n),\mu^+,\tau)=P(y_u|\mu,n)-P(y_l|\mu,n)
\]
where $P(y\mid\mu,n)$ is the predictive CDF and
\[
y_l=\frac{(\mu^+-\delta/2)(n+\tau)-n\mu}{\tau}\quad\text{ and }\quad
y_u=\frac{(\mu^++\delta/2)(n+\tau)-n\mu}{\tau}.
\]
As $P$ is Gaussian the transition probabilities are fast to calculate.

Transitions to states with $\mu^+\not\in\Omega$ are dealt with by treating such
states as terminal states with no further recursion. For $\mu^+<\min\Omega$ we
retire to the safe arm (value is zero). This gives no approximation error in the
resulting GI (see Appendix \ref{sec:proof}). For $\mu^+>\max\Omega$ a value
function approximation which assumes no further learning is used. This does cause an
underestimation in $\nu^{GI(\tau)}$ but for $\xi$ sufficiently high the
approximation will be small as the probability of $\theta_a>\max\Omega$ is
very low.

Adapting \eqref{equ:gi_V_N} with $N$ appropriate to the OAB stage, the
approximation for transitions to states with $\mu^+>\max\Omega$ is
\[
V^+=\mathbb{E}_{Y}\left[V_N(\mu^+,n+\tau,\gamma,\lambda\right].
\]
Our complete value function approximation (given $\Omega$) is then
\begin{align*}
\hat V(\mu,&n,\gamma,\tau,\lambda)\\
&=\max\left\{\mu-\lambda+\gamma\left[V^++\sum_{\mu^+\in\Omega}\Big[\pi[(\mu,n),\mu^+,\tau]\hat
V(\mu^+,n+\tau,\gamma,\lambda)\Big]\right];0\right\}.
\end{align*}

The value function $\hat V$ retains the monotonicity properties of $V$ with
$\lambda$, $\mu$ and $n$. A major efficiency saving can be made by exploiting
the following results which follow directly from the monotonicity of $V$
with $\mu$ and $n$.
\begin{proposition}
\label{prop1}
For all $\tmu<\mu$,
$V(\mu,n,\gamma,\tau,\lambda)=0\implies V(\tmu,n,\gamma,\tau,\lambda)=0$.
\end{proposition}
\begin{proposition}
\label{prop2}
For all $\tn<n$,
$V(\mu,n,\gamma,\tau,\lambda)>0\implies V(\mu,\tn,\gamma,\tau,\lambda)>0$.
\end{proposition}
These are used in the following manner. For each $n$, $\hat V$ is calculated for states in order of decreasing
$\mu$, then, by Proposition \ref{prop1}, as soon as the safe arm is chosen it
can be chosen for all remaining states with the same $n$. Similarly, working
backwards through $n$, as soon as $V(\mu_a,n,\gamma,\tau,\lambda)$ is greater
than zero, we know from Proposition \ref{prop2} that
$V(\mu_a,n_a,\gamma,\tau,\lambda)>0$ and there is no need for any further
calculation. Together, these results reduce the number of states for which
$\hat V$ needs to be calculated.

The calculation time and the effect of the approximations $N$, $\delta$ and
$\xi$ on accuracy of $\nu^{GI(\tau)}$ can be seen in convergence tests in
Appendix \ref{sec:convergence}. The approximations from too small an $N$ or
$\xi$ cause an underestimation in $\nu^{GI}$ while too high a $\delta$ causes an overestimation.  The
approximation error due to $N$ is smaller than for the BMAB and is very
manageable, as is the error due to $\xi$ which is minimal with $\xi=3$. The
guarantees for $\delta$ are less clear than with $N$ and $\xi$ and the choice of
$\delta$ can have a large effect on the run time of the algorithm. Overall,
better than 3 decimal place accuracy for $\gamma\leq0.99$ can be obtained by
using $N=140$, $\delta=0.01$ and $\xi=3$. Calculation with these settings takes
12 seconds for each $n$. Note that since the same settings are used for all arms
the errors will be correlated and so the differences in GI between arms will
usually be smaller than in the study.

\section{Discussion}
\label{sec:gi_discuss}
This paper gives simple methods, with accompanying code, to easily calculate GIs
for an extensive range of states for the BMAB and the NMAB. The methodology
given can be used to find GIs for other common MAB problems with exponential
family rewards, for example those with exponential, Poisson and Binominal
rewards. This paper did not discuss the calculation of GIs for more general
reward distributions but the same ideas can be used. The greatest
difficulty in calculation comes when states have multiple continuous
parameters which cannot be reduced using invariance results as was done with the
NMAB.

The only area of the BMAB and NMAB that could remain difficult is for $\gamma$
close to 1. To guarantee good accuracy the horizon $N$, and therefore the state
space of the OAB, will be large. In addition, with higher $\gamma$, MABs have
longer effective time horizons before rewards become small and so $\nu^{GI}$
will be needed for more states. Even here, though, the difficulties should not
be overstated. For large $N$ in the OAB, or for large $T$ in the MAB, the
posteriors of arms narrow around $\theta_a$ so that terminal states
approximations are very good. Therefore GIs can still be calculated to good
accuracy and will produce a policy that is far closer to Bayes' optimality than heuristic alternatives.

The methods given here can also be used to calculate some forms of Whittle
indices, a generalisation of GIs, for example in the common MAB variant where
the horizon is finite. A finite horizon adds an extra variable $s$, the time
remaining to the end of the horizon and requires a different set of Whittle
indices $\nu^{WI}(\Sigma,n,\gamma,s)$ for each possible $s$. Each set of states
is smaller than would be needed for GI since only a limited range of $n$ can be
reached given $s$ and prior $n_0$. For each state,
$\nu^{WI}$ can be calculated as for GI with the advantage
that $N$ does not cause any approximation and will often be small.
The storage costs may be greater than for the standard MAB, but otherwise
calculating indices for the finite horizon MAB problem poses little extra
difficulty compared to GIs.

\section{Acknowledgements}
The author was supported by the EPSRC funded EP/H023151/1 STOR-i CDT. The author
would like to thank Kevin Glazebrook, Paul Fearnhead, Richard Weber and Peter
Jacko for their comments and feedback.

\bibliographystyle{agsm}
\bibliography{gittins}
\begin{appendices}
\section{Accuracy of a Finite Horizon Approximation for the BMAB Value Function}
\label{sec:finite_N}
Table \ref{tab:gi_bmab_N} gives the results from a convergence test for the BMAB. The
comparison is made against $\nu^{GI}$ calculated using $N=2000$. Similar experiments in Section \ref{sec:convergence} show convergence
for the NMAB at even smaller $N$.
\begin{table}[H]
\centering
\begin{tabular}{|r|r|r|r|r|}
\hline
\multicolumn{1}{|c|}{\multirow{2}{*}{N}} & \multicolumn{2}{c|}{$\gamma=0.9$}                     & \multicolumn{2}{c|}{$\gamma=0.99$}    \\ \cline{2-5} 
\multicolumn{1}{|c|}{}                   & \multicolumn{1}{c|}{Error} & \multicolumn{1}{c|}{RRN} & \multicolumn{1}{c|}{Error} & RRN      \\ \hline
20                                       & 0.00827                    & 1.21577                  & 0.03738                    & 81.79069 \\
60                                       & 0.00010                    & 0.01797                  & 0.02825                    & 54.71566 \\
100                                      & 0                          & 0.00027                  & 0.01755                    & 36.60323 \\
200                                      & 0                          & 0.00000                  & 0.00557                    & 13.39797 \\
400                                      & 0                          & 0.00000                  & 0.00066                    & 1.79506  \\
800                                      & 0                          & 0.00000                  & 0.00001                    & 0.03222  \\\hline
\end{tabular}
\caption{The error in BMAB $\nu^{GI}$ from using a finite $N$ with
$\Sigma=1$, $n=2$, accuracy $\epsilon=5\times10^{-6}$ and $N,\gamma$ as shown.
RRN is the maximum reward remaining after stage $N$.}
\label{tab:gi_bmab_N}
\end{table}

\section{Results of Convergence Tests for NMAB GI}
\label{sec:convergence}
This section gives the results of convergence tests for the NMAB investigating
the effect of the approximations due to $N$, $\delta$ and $\xi$. For each
$\gamma\in\{0.9,0.99\}$ a benchmark value of $\nu^{GI(\tau)}(0,1,\gamma,1)$ was
calculated using $N=200$, $\xi=6$, $\delta=0.005$ and accuracy
$\epsilon=5\times10^{-5}$. Then the values of $N$, $\delta$ and $\xi$ were
relaxed individually in turn. The difference $\nu^{GI}-\hat\nu^{GI}$ between the
the benchmark GI value and the approximation are given in Tables
\ref{tab:nmab_N} to \ref{tab:nmab_xi} together with the calculation times.

\begin{table}[H]
\centering
\begin{tabular}{|c|c|c|c|}
\hline
N & Time (s) & Error: $\gamma = 0.9$ & Error: $\gamma=0.99$ \\ \hline
20 & 58 & 0.0001 & 0.0084 \\
40 & 94 & 0 & 0.0022 \\
60 & 127 & 0 & 0.0008 \\
80 & 156 & 0 & 0.0004 \\
100 & 183 & 0 & 0.0002 \\
120 & 210 & 0 & 0.0001 \\
140 & 235 & 0 & 0 \\ \hline
\end{tabular}
\caption{Approximation error and run time of GI calculations with
varying $N$ and $\gamma$ as shown and fixed $n=1$, $\tau=1$,
$\epsilon=5\times10^{-5}$, $\delta=0.005$, $\xi=6$.}
\label{tab:nmab_N}
\end{table}

\begin{table}[H]
\centering
\begin{tabular}{|c|c|c|c|}
\hline
\multirow{2}{*}{$\delta$} & \multirow{2}{*}{Time (s)} & \multicolumn{2}{c|}{Error} \\ \cline{3-4}
 &  & $\gamma = 0.9$ & $\gamma=0.99$ \\ \hline
0.08 & 3 & -0.0011 & -0.0059 \\
0.04 & 7 & -0.0002 & -0.0023 \\
0.02 & 24 & 0 & -0.0007 \\
0.01 & 89 & 0 & -0.0002 \\ \hline
\end{tabular}
\caption{Approximation error and run time of GI calculations with
 varying $\delta$ and $\gamma$ as shown and fixed $n=1$, $\tau=1$, $N=200$,
 $\epsilon=5\times10^{-5}$, and $\xi=6$.}
\label{tab:nmab_delta}
\end{table}

\begin{table}[H]
\centering
\begin{tabular}{|c|c|c|c|}
\hline
\multirow{2}{*}{$\xi$} & \multirow{2}{*}{Time (s)} & \multicolumn{2}{c|}{Error} \\ \cline{3-4}
 &  & $\gamma = 0.9$ & $\gamma=0.99$ \\ \hline
2 & 24 & 0 & 0.0109 \\
2.5 & 46 & 0 & 0.0002 \\ 
3 & 73 & 0 & 0 \\ \hline
\end{tabular}
\caption{Approximation error and run time of GI calculations with
varying $\xi$ and $\gamma$ as shown and fixed $n=1$, $\tau=1$, $N=200$,
$\epsilon=5\times10^{-5}$, and $\delta=0.005$.}
\label{tab:nmab_xi}
\end{table}

\section{Proposition for Retirement when $\mu<\mu_a$}
\label{sec:proof}
This considers an approximation to the value function \eqref{equ:value_tau}
which retires (sets $V$ to 0) in any state where $\mu<\mu_a$. 
\begin{proposition}
Let 
\begin{align*}
\tV(\mu,n,\gamma,\tau,\lambda;\mu_a)=
\begin{cases}
\max\left\{\mu-\lambda+
\gamma\mathbb{E}_{Y}\big[\tV(\mu^+\mid Y,n+\tau,\gamma,\lambda;\mu_a) \big]
;0\right\},&\mbox{if }\mu\geq\mu_a\\
 0,&\mbox{otherwise}.
\end{cases}
\end{align*}
The value of $\nu^{GI(\tau)}(\mu_a,n_a,\tau)$ is the same whether
$V$ or $\tV$ is
used for the calibration. That is,
$V(\mu_a,n_a,\gamma,\tau,\lambda)=0\iff\tV(\mu_a,n_a,\gamma,\tau,\lambda;\mu_a)=0$.
\end{proposition}
\begin{proof}
First note that both $\tV$ and $V$ are always non-negative and are
non-increasing in $n$ and non-decreasing in $\mu$. The proof will compare the
values assigned to states in the dynamic programme for $V$ and $\tV$. 

The terminal states at stage $N$ are the same
for each function. Working backwards, let $n^*$ be the smallest $n$ (earliest
stage) where $V(\mu_a,n,\gamma,\tau,\lambda)>0$. By the monotonicity properties,
$V(\mu,n,\gamma,\tau,\lambda)=0$ for all $\mu\leq\mu_a$, $n>n^*$ so
$\tV(\mu,n,\gamma,\tau,\lambda;\mu_a)=V(\mu,n,\gamma,\tau,\lambda)$ for all
states where $n>n^*$, by definition. Since the calculation for state
$(\mu_a,n^*)$ depends only on these subsequent states and the calculation for
states with $\mu=\mu_a$ is the same for each value function, we have
$\tV(\mu_a,n^*,\gamma,\tau,\lambda;\mu_a)=V(\mu_a,n^*,\gamma,\tau,\lambda)>0$.
Therefore, by the monotonicity with $n$ for both functions, both
$\tV(\mu_a,n_a,\gamma,\tau,\lambda;\mu_a)>0$ and
$V(\mu_a,n_a,\gamma,\tau,\lambda)>0$ and the proposition is satisfied in this
case. If no such $n^*$ exists then then both value functions give identical
values for each state and so the value in the starting state will be identical for each.
\end{proof}

\end{appendices}

\end{document}